\documentclass{amsart}

\usepackage{xcolor}
\usepackage{amssymb}

\usepackage{graphicx,wrapfig}
\usepackage{caption}
\usepackage[nopar]{lipsum}

\usepackage[linesnumbered,ruled,vlined]{algorithm2e}
\SetKwInput{KwInput}{Input}                
\SetKwInput{KwOutput}{Output}              

\newtheorem{theorem}{Theorem}
\newtheorem{corollary}{Corollary}

\theoremstyle{definition}

\theoremstyle{remark}

\numberwithin{equation}{section}

\usepackage{eqnarray,amsmath}

\newcommand{\N}{\mathbb{N}}
\newcommand{\R}{\mathbb{R}}
\newcommand{\Z}{\mathbb{Z}}

\newcommand{\abs}[1]{|#1|}
\newcommand\norm[1]{\left|#1\right|}

\renewcommand{\dot}[2]{\left\langle#1,#2\right\rangle}

\renewcommand{\L}{\mathcal{L}}

\usepackage{mathrsfs}
\usepackage[scr=boondoxo]{mathalfa}
\newcommand{\gt}[1]{\mathscr{#1}}

\usepackage[numbers,sort,compress]{natbib}

\begin{document}

\title{Understanding Sinusoidal Neural Networks}

\author[T. Novello]{Tiago Novello \\ IMPA}
\address{IMPA - VISGRAF, Rio de Janeiro, Brazil}
\email{tiago.novello90@gmail.com}


\date{\today~\hfill \\ 
\indent \textit{E-mail address}: \texttt{tiago.novello90@gmail.com}}


\begin{abstract}
In this work, we investigate the structure and representation capacity of \textit{sinusoidal} MLPs --- multilayer perceptron networks that use \textit{sine} as the activation function.
These neural networks (known as \textit{neural fields}) have become fundamental in representing common signals in computer graphics, such as images, signed distance functions, and radiance fields. 
This success can be primarily attributed to two key properties of sinusoidal MLPs: \textit{smoothness} and \textit{compactness}.
These functions are smooth because they arise from the composition of affine maps with the sine function. 
This work provides theoretical results to justify the compactness property of sinusoidal MLPs and provides control mechanisms in the definition and training of these networks.

We propose to study a sinusoidal MLP by expanding it as a \textit{ harmonic}~sum.
First, we observe that its first layer can be seen as a harmonic dictionary, which we call the \textit{input $($sinusoidal$)$ neurons}.
Then, a hidden layer combines this dictionary using an affine map and modulates the outputs using the sine, this results in a special dictionary of \textit{sinusoidal neurons}.
We prove that each of these sinusoidal neurons expands as a harmonic sum \textit{producing} a large number of new \textit{frequencies} expressed as integer linear combinations of the input frequencies. 
Thus, each hidden neuron produces the same frequencies, and the corresponding \textit{amplitudes} are completely determined by the hidden affine map. 
We also provide an upper bound and a way of sorting these amplitudes that can control the resulting approximation, allowing us to truncate the corresponding series.
Finally, we present applications for training and initialization of sinusoidal MLPs. Additionally, we show that if the input neurons are periodic, then the entire network will be periodic with the same period. We relate these \textit{periodic networks} with the Fourier series representation. 
\end{abstract}
\maketitle

\section{Introduction}
\label{s-introduction}

\textit{Neural fields} is a research topic that involves representing a graphical object, such as an image or surface, as a neural network $f:\R^n\to \R^m$. 
This network maps each coordinate $x$ of the domain $\R^n$ (e.g., 2D or 3D points) to its corresponding value $f(x)$ in the codomain $\R^m$ (e.g., color or distance from a surface). \textit{Multilayer perceptron networks} (MLPs) are important examples.

The use of the sine activation function in MLPs has attracted the attention of the neural field community~\cite{parascandolo2016taming, sitzmann2020implicit, paz2022mrnet, paz2023mr, novello21diff, novello21neuralAnimation, benbarka2022seeing, yang2021geometry, mehta2022level, da2022neural, schardong2023neural}.
However, determining the parameters (such as width) of these MLPs remains an empirical task.
This work considers an alternative approach to studying this problem through a novel expansion formula of a sinusoidal neuron in terms of a sum of sines (Theorem~\ref{t-siren2sum}). 
This formula resembles a \textit{Fourier series} of the neuron.

In this work, we approach the task of expanding a \textit{sinusoidal} MLP $f:\R\to~\R$ with a single \textit{hidden layer} $\textbf{h}:\R^n\to \R^n$, with \textit{width} $n\in\N$, as a sum of sines.
Specifically, we consider $f$ to be expressed as:
\begin{align}\label{e-siren}
f(x)=\textbf{l}\circ \textbf{h} \circ \textbf{s}(x)+d,    
       \end{align}
with $\textbf{h}(x)=\sin(\textbf{a} x+b)$, where $\textbf{a}\in\R^{n\times n}$ is the \textit{hidden matrix}, and $b\in\R^{n}$ is the \textit{bias}. 
The first layer $\textbf{s}:\R\to \R^n$ is defined as $\textbf{s}(x)=\sin(\omega x+\varphi)$, with $\omega, \varphi \in\R^n$. 
Finally,  $\textbf{l}(x)=\dot{c}{x}+d$ is a linear layer with $c\in\R^n$ and $d\in\R$.

We justify the assumption that the network $f$ has $\mathbb{R}$ as its domain and codomain.
First, note that the analysis of a sinusoidal MLP $f\!:\!\mathbb{R}^k\!\to\!\mathbb{R}^m$ reduces to the analysis of the $m$ MLPs $f_i:\mathbb{R}^k\to \mathbb{R}$ corresponding to the coordinates of $f$.
Moreover, the first layer $\textbf{s}(x)$ of a sinusoidal MLP $f:\mathbb{R}^k\to \mathbb{R}$ is expressed as $\textbf{s}(x) = \sin\left(\sum_{i=1}^k \omega_i x_i+\varphi\right)$, where $\omega_i$ and $x$ are in $\mathbb{R}^k$.
Therefore, when restricting $f$ to the coordinate $x_1$ results in a sinusoidal MLP of the form $\mathbb{R}\to \mathbb{R}$ with its first layer defined as
$\sin\left(\omega_1 x_1 + \sum_{i=2}^k \omega_i x_i+\varphi\right)$,
with $\sum_{i=2}^k \omega_i x_i+\varphi$ being its first bias. The same applies to the other coordinates.
Hence, throughout this work, we assume that the sinusoidal MLPs have the form $\mathbb{R}\to \mathbb{R}$.

\vspace{0.2cm}

Previously, sinusoidal MLPs have been regarded as difficult to train~\cite{parascandolo2016taming}. Sitzmann et al.~\cite{sitzmann2020implicit} overcome this by defining a special initialization scheme that guarantees stability and convergence.
This motivated several works that proved empirically that these networks have a high capacity for representing fine details.
This success can be attributed to two key properties of sinusoidal MLPs: \textit{smoothness} and \textit{compactness}.
These functions are smooth because they arise from the composition of affine maps with the sine function. 
Specifically, the layers $\textbf{l}$, $\textbf{h}$, $\textbf{s}$ are smooth because they are affine maps composed with the sine function, thus, their composition is also smooth.
In this work, we provide theoretical results to justify the compactness property of sinusoidal MLPs and give control mechanisms in defining and training the network architecture.
For this, we propose to study a sinusoidal MLP by expanding it as a \textit{harmonic} sum.

We start with an interpretation of the first and second layers of a sinusoidal MLP $f$ in terms of harmonics.
The first layer $\textbf{s}$ can be interpreted as a dictionary of ({harmonics}) \textit{input sinusoidal neurons} $\textbf{s}_i(x)=\sin(\omega_ix+\varphi_i)$ as follows:
\begin{align}
\displaystyle
    \textbf{s}(x)=
    \left(
    \begin{array}{c}
        \sin(\omega_1x+\varphi_1)\\
         \vdots\\
         \sin(\omega_nx+\varphi_n)
    \end{array}
    \right)
\end{align}
For simplicity, we may omit the \textit{sinusoidal} term and refer to $\textbf{s}_i$ as \textit{input neurons}.
Thus, the first weight matrix $\omega$ gives the input \textit{frequencies} and the first bias $\varphi$ the \textit{phase shifts} of the input neurons $\textbf{s}_i$.
The following equation says that the bias $\varphi$ is responsible of including the cosine functions $\cos(\omega_i x)$ into the dictionary: 
$$\textbf{s}_i(x)=\sin(\omega_ix+\varphi_i)=\cos(\varphi_i)\sin(\omega_ix)+\sin(\varphi_i)\cos(\omega_ix).$$
For an example, consider $\varphi=(0,\ldots, 0, \frac{\pi}{2}, \ldots, \frac{\pi}{2})\in \R^{2n}$ with the first $n$ coordinates being zeros and the remaining being $\frac{\pi}{2}$. Now, defining the first weight matrix by $2\pi (\omega, \omega)$, results that $\textbf{s}(x)$ represents a \textit{Fourier mapping}~\cite{tancik2020fourier}:
\begin{align}
\displaystyle
    \textbf{s}(x)=
    \left(
    \begin{array}{c}
        \sin(2\pi\omega x)\\
         \cos(2\pi\omega x)
    \end{array}
    \right).
\end{align}
\citet{benbarka2022seeing} explored this fact and initialized the first matrix using integer frequencies ($\omega\subset \Z$) to relate $\textbf{s}(x)$ with Fourier series.
For example, assume that $\omega$ has distinct integer coordinates, i.e $\omega_i\in\Z$ and $\omega_i\neq \omega_j$~for~$i\neq~\!j$, thus $\textbf{s}(x)$ is a list of orthogonal functions generated by a finite subset of the Fourier basis: $\sin(0 x), \sin(2\pi i x), \cos(2\pi i x)$ with $i\in\N$.
However, only the first layer $\textbf{s}$ was considered in the analysis of~\cite{benbarka2022seeing}. In this work, we prove that the whole network $f$ can be expanded in harmonic sum similar to a Fourier series (see Theorem~\ref{t-network2sum}).

\vspace{0.1cm}
We now return to the generic case. The hidden layer $\textbf{h}$ of $f$ 
combines the input neurons $\textbf{s}_i$ with \textit{amplitudes} determined by the hidden matrix $\textbf{a}\in\R^{n\times n}$, resulting in the following list of (sinusoidal) \textit{hidden neurons} $\textbf{h}_i$:
\begin{align}\label{e-second_layer}
    \textbf{h}\big(\textbf{s}(x)\big)=
    \left(
    \begin{array}{c}
        \displaystyle\sin\left(\sum_{j=1}^n a_{1j}\sin(\omega_jx+\varphi_j)+b_1\right)\\
         \vdots\\
         \displaystyle\sin\left(\sum_{j=1}^n a_{nj}\sin(\omega_jx+\varphi_j)+b_n\right)
    \end{array}
    \right).
\end{align}
That is, $\textbf{h}\big(\textbf{s}(x)\big)$ is modulation of the harmonic sums $\sum_{j} a_{ij}\textbf{s}_j(x)+b_i$ by the sine function resulting in the list of \textit{sinusoidal neurons} $\textbf{h}_i(x)=\sin\left(\sum_{j} a_{ij}\textbf{s}_j(x)+b_i\right)$. The \textit{hidden biases} $b_i$ can be seen as the $0^{th}$ \textit{harmonics} of this sum.
Observe that we interpret each hidden neuron as a function $\textbf{h}_i:\R\to \R$.

\vspace{0.1cm}

In this work, we prove that each sinusoidal neuron $\textbf{h}_i$ expands as a harmonic sum with its frequencies completely determined by special (integer) linear combinations, $\sum_{i=1}^n k_i\omega_i$ with $k_i\in\Z$, of the input frequencies~$\omega_i$.
For this, we derive a formula (Theorem~\ref{t-siren2sum}) that expresses, in closed form, the amplitudes, frequencies, and phase shifts of this sum. 
In other words, this novel \textit{trigonometric identity} can be used for an analytical derivation of the Fourier spectra of the sinusoidal neuron $\textbf{h}_i$.
We explore this sinusoidal expansion in the following applications.
\begin{itemize}
\item We prove that the sinusoidal neurons $\textbf{h}_i$ produce a large number of new frequencies that are the same for each neuron (Sec~\ref{s-sinusoidal_neurons}). 
However, each $\textbf{h}_i$ defines different \textit{amplitudes} for the corresponding harmonics, and those are determined by the line $a_i$ of the hidden matrix $\textbf{a}$ (Sec~\ref{s-sinusoidalMLP}). 

\item We provide an upper bound and a way of sorting the amplitudes in the expansion of a sinusoidal neuron  (Secs~\ref{s-upperbound}~and~\ref{s-sorting}). Such mechanisms can be used to control the resulting approximation, allowing us to truncate the corresponding series.

\item We explore this Fourier-like interpretation of sinusoidal MLPs to analyze their behavior during training and provide adequate initialization schemes (Secs~\ref{s-training} and \ref{s-initialization}). 

\item  Finally, we show that our trigonometric formula implies that if the input neurons are periodic, then the entire neural network will be periodic with the same period (Sec~\ref{s-periodicity}). We relate these \textit{periodic networks} with the Fourier series representation. 
\end{itemize}

\section{Sinusoidal Neurons}
\label{s-sinusoidal_neurons}
In this section, we show that each sinusoidal neuron expands as a sum of harmonics with their \textit{frequencies} completely determined by the input frequencies $\omega$ and \textit{amplitudes} by the coefficients of the hidden matrix $\textbf{a}$.

Precisely, we define a \textit{sinusoidal neuron} as a function $h:\R\to \R$ expressed as 
$$h(x)=\sin\left(\sum_{i=1}^n a_i\sin(\omega_ix+\varphi_i)+ b\right)$$
where $\textbf{a}$, $\omega$, $\varphi$, $\textbf{b}$ are the amplitudes, frequencies, phases, and bias, respectively.
Note that, following the common machine learning neuron notation, $h$ receives a list of $n$ {input neurons} $\textbf{s}_i(x)=\sin(\omega_ix+\varphi_i)$ that are combined with the weights $a_i$ and activated by $\sin$. The number $n$ is the \textit{width} of the sinusoidal neuron $h$.

\vspace{0.1cm}

Before presenting the expansion of a sinusoidal neuron with width $n$, let us recall the (Fourier) expansion of a neuron with width $1$ and no bias~\cite[Page 361]{abramowitz1964handbook}:
\begin{align}\label{e-fourier-expansion-composition}
    \sin\big(a\sin(x)\big)\!=\!\!\sum_{k\in\Z \text{ odd}} \!\!\!J_k(a) \sin(kx),\,\,\text{with}\,\, J_k(a)\!=\!\!\int_0^\pi\!\!\!\cos\big(kt-x\sin(t)\big) dt.
\end{align}
The functions $J_k$ are the well-known \textit{Bessel functions} of the first kind.
To provide an expansion of the sinusoidal neuron $h$ we must generalize the formula in Equation~\ref{e-fourier-expansion-composition}. 
For this, we prove the following result using an inductive argument on the width $n$ of the sinusoidal neuron $h$.

\begin{theorem}
\label{t-siren2sum}
A sinusoidal neuron $h$ with width $n$ expands as a harmonic sum
\begin{align}\label{e-perceptron_approx}
h(x)= \sum_{\textbf{k}\in\Z^n}\alpha_\textbf{k}(a)\sin\Big(\dot{\textbf{k}}{\omega x +\varphi}+ b\Big) \quad \text{with}\quad \alpha_\textbf{k}(a)=\prod_{i=1}^n J_{k_i}(a_i).
\end{align}
\end{theorem}

Before presenting the proof of Theorem~\ref{t-siren2sum} let's provide some additional details.
Observe that the \textit{frequencies} and \textit{phase-shifts} in the expansion given in Equation~\ref{e-perceptron_approx} are integer linear combinations of the input frequencies $\omega_i$: 
\begin{align}\label{e-new_frequencies}
    \beta_\textbf{k}(\omega)=\sum_{i=1}^{n} k_i\omega_{i}, \quad  \lambda_\textbf{k}(\varphi, b)=\sum_{i=1}^{n} k_i\varphi_{i}+b.
\end{align} 
As a consequence, we note that the activation of the sinusoidal neuron sum $\textbf{s}_i$ by the sine function is \textit{producing} a lot of frequencies $\beta_\textbf{k}(\omega)$ in terms of the input frequencies~$\omega$. 
More precisely, truncating the expansion, that is, summing over $\norm{\textbf{k}}_\infty\leq B\in \N$, implies that the neuron $h$ can learn $\frac{(2B+1)^n -1}{2}$ non-null frequencies.
To compute this number, we remove the case $\textbf{k}=0$ and count $\{\textbf{k},-\textbf{k}\}$ as a single case.
Also, note that the $0^{th}$ harmonic is present in the expansion of $h$ because $\beta_{\textbf{0}}(\omega)=0$, and it corresponds to $\prod_{i=1}^n J_{0}(a_i)\sin(b)$. 
This makes the last bias in a sinusoidal neural network probably not necessary.

A natural issue regarding such frequency factoring is the existence of a $B$ such that $\norm{\alpha_\textbf{k}(a)}$ is small for $\norm{\textbf{k}}_\infty\leq B$.
In Section~\ref{s-upperbound}, we provide an upper bound for $\alpha_\textbf{k}(a)$ that implies a rapid decrease of the $\textbf{k}$-amplitudes as $\textbf{k}$ increases.

\vspace{0.2cm}

{This frequency factoring explains why composing sinusoidal layers \textit{may} compact data information}. Specifically, from the above discussion, we deduced that the sinusoidal neuron $h$ receiving $n$ frequencies, parameterized by $3n+1$ coefficients, can be used to approximate a harmonic sum with $\frac{(2B+1)^n -1}{2}$ terms.

Another important observation is that the weights $\textbf{a}$ fully determine the amplitude $\alpha_\textbf{k}(\textbf{a})$ of each harmonic term $\sin\big( \beta_\textbf{k}(\omega) x+\lambda_\textbf{k}(\varphi, b) \big)$ in the expansion. Consequently, the input matrix $\omega$ determines the frequencies that can be represented by $h$, while the hidden matrix $\textbf{a}$ determines which frequencies will be used.

\vspace{0.2cm}
Back to the proof of Theorem~\ref{t-siren2sum}. It consists of proving the formula in Equation~\ref{e-perceptron_approx} as well as a similar one using the cosine as the activation function:
\begin{align}\label{e-perceptron_approx_cosine}
\cos\left(\sum_{i=1}^n a_i\sin(\omega_ix+\varphi_i)+ b\right)= \sum_{\textbf{k}\in\Z^n}\alpha_\textbf{k}(a)\cos\Big(\dot{\textbf{k}}{\omega x +\varphi}+ b\Big).
\end{align}

\begin{proof}[Proof of Theorem~\ref{t-siren2sum}]
The proof is by induction in $n$. For the base case $n=1$, we prove $\sin\left(a \sin(y)+b\right)=\sum_{k\in \Z} J_k(a) \sin(ky+b)$ with $y=\omega x+ \varphi$.
For this, we use the expansion in Equation \ref{e-fourier-expansion-composition} and its cosine analogous expansion $\cos\big(a\sin(y)\big)\!=\!\!\sum J_l(a) \sin(ly)$, here the sum is over the even numbers. Thus, applying the trigonometric identity $\sin(a+b)=\sin(a)\cos(b)+\cos(a)\sin(b)$ we obtain:
\begin{align*}
    \sin\big(a \sin(y)+b\big)&=\sin\big(a \sin(y)\big)\cos(b)+\cos\big(a \sin(y)\big)\sin(b)\\
    &=\sum_{k\in\Z \text{ odd}} \!\!\!J_k(a) \sin(ky)\cos(b) + \sum_{l\in\Z \text{ even}} \!\!\!J_l(a) \cos(ly)\sin(b)\\
    &=\sum_{k\in\Z \text{ odd}} \!\!\!J_k(a) \sin(ky+b) + \sum_{l\in\Z \text{ even}} \!\!\!J_l(a) \sin(ly+b)\\
    &=\sum_{k\in\Z} J_k(a) \sin(ky+b).
\end{align*}
In the third equality we combined the formula $\sin(a)\cos(b)=\frac{\sin(a+b)+\sin(a-b)}{2}$ and the fact that $J_{-k}(a)=(-1)^kJ_k(a)$ to rewrite the summations.
The proof of the formula using the cosine as activation function is similar.

Assume that the formulas hold for $m-1$, with $m>1$, we prove that Equation~\ref{e-perceptron_approx} holds for $m$ (the induction step). Again, we denote $y_i=\omega_i x+ \varphi_i$ for simplicity.

\begin{align*}
    \sin\left(\sum_{i=1}^{m} a_i\sin(y_i)+ b\right)&=\sin\left(\sum_{i=1}^{m-1} a_i\sin(y_i)+ b\right)\cos\big(a_m\sin(y_m)\big)\\
    &+\cos\left(\sum_{i=1}^{m-1} a_i\sin(y_i)+ b\right)\sin\big(a_m\sin(y_m)\big)
    \\
    &=\sum_{\textbf{k}\in\Z^{m-1},\, l\in\Z \text{ even}}\alpha_\textbf{k}(a)J_l(a_m)\sin\Big(\dot{\textbf{k}}{y}+ b\Big)\cos(ly_m)\\
    &+\sum_{\textbf{k}\in\Z^{m-1},\, l\in\Z \text{ odd}}\alpha_\textbf{k}(a)J_l(a_m)\cos\Big(\dot{\textbf{k}}{y}+ b\Big)\sin(ly_m)\\
    &=\sum_{\textbf{k}\in\Z^m}\alpha_\textbf{k}(a)\sin\Big(\dot{\textbf{k}}{y}+ b\Big)
\end{align*}
We use the induction hypothesis in the second equality and an argument similar to the one used in the base case to rewrite the harmonic sum. Again, the cosine activation function case is analogous.
\end{proof}

\citet{yuce2022structured} presented a similar formula for MLPs activated by polynomial functions.
While it is evident that the sine function can be approximated by a polynomial using Taylor series, our formula requires no approximations.
In addition to providing a simple proof, we also derive the analytical expressions for the amplitudes. These expressions enable us to compute upper bounds for the new frequencies (Theorem~\ref{t-upper_bound}).
Furthermore, we apply our formula to show that sinusoidal MLPs are a very close representation of the Fourier series. Finally, we use it to prove a periodicity theorem concerning sinusoidal MLPs (Theorem~\ref{t-periodic}).

\vspace{0.2cm}

We notice that the expansion of the sinusoidal neuron $h$ in Equation~\ref{e-perceptron_approx} resembles the \textit{amplitude-phase form} of the Fourier series.
On the other hand, the \textit{sine-cosine} and \textit{exponential forms} follow directly as corollaries of Theorem~\ref{t-siren2sum}.

\begin{corollary}\label{c-fourier_expansion}
    A sinusoidal neuron $h(x)=\sin\Big(\sum_{i=1}^n a_i\sin(\omega_ix+\varphi_i)+ b\Big)$ expands as sum of sines and cosines: 
    \begin{align}\label{e-sine-cosine-form}
        h(x)=\sum_{\textbf{k}\in\Z^n}A_\textbf{k}\cos\Big(\dot{\textbf{k}}{\omega} x\Big) + B_\textbf{k}\sin \Big(\dot{\textbf{k}}{\omega} x\Big)
    \end{align}
    with $A_\textbf{k}=\alpha_\textbf{k}(a)\sin\Big(\dot{\textbf{k}}{\varphi} +b\Big)$ and $B_\textbf{k}=\alpha_\textbf{k}(a)\cos\Big(\dot{\textbf{k}}{\varphi} +b\Big)$.
    
\noindent We can also rewrite this sine-cosine sum as a sum of complex exponentials:    
        \begin{align}\label{e-exponential-form}
        h(x)=\sum_{\textbf{k}\in\Z^n}c_\textbf{k}\text{e}^{i\dot{\textbf{k}}{\omega}x},  \text{ with } c_\textbf{k} = i\alpha_\textbf{k}(a)\text{e}^{i\dot{\textbf{k}}{\varphi}}\left(\frac{\text{e}^{ib}-(-1)^{\sum k_i} \text{e}^{-ib}}{2}\right).
    \end{align}
\end{corollary}

For simplicity, we have used the notation $A_\textbf{k}$ instead of $A_\textbf{k}(a, b, \varphi)$; the same applies to $B_\textbf{k}$ and $c_\textbf{k}$.
Observe that the sine-cosine and exponential forms only resemble the Fourier expansion of the neuron $h$, as the frequencies $\dot{\textbf{k}}{\omega}$ could not be integer multiples of $2\pi$. However, we can initialize the coordinates of $\omega$ as integer multiples of $2\pi$, which could be used to rearrange the above sum as the Fourier series of $h$.
In Section~\ref{s-periodicity}, we provide further details about such initialization.

\subsection{Upper bound for the amplitudes}
\label{s-upperbound}

Here we show that the expansion of the sinusoidal neuron $h$ can be truncated by a (small) integer $B>0$:
$$h(x)\approx \sum_{\norm{\textbf{k}}_{\infty}\leq B}\alpha_\textbf{k}(a)\sin\Big(\dot{\textbf{k}}{\omega x +\varphi}+ b\Big).$$
Where $\norm{\textbf{k}}_{\infty}=\max\{|k_i|\}$.
To this end, we study how the amplitudes $\alpha_\textbf{k}(a)$ behave as $B$ grows. 
This is presented in the following result which gives an upper bound for the amplitudes $\alpha_\textbf{k}(a)$ in terms of $\textbf{k}$ and ${a}$.

\begin{theorem}\label{t-upper_bound}
The amplitude $\alpha_\textbf{k}(a)$ associated with the frequency $\dot{\textbf{k}}{\omega}$ in the harmonic expansion of a sinusoidal neuron is bounded as follows:
\begin{align}\label{e-upper-bound-freq}
    \abs{\alpha_\textbf{k}(a)}<
    \prod_{i=1}^n\frac{\left(\frac{\abs{a_i}}{2}\right)^{\norm{k_i}}}{\abs{k_i}!}.
\end{align}
\end{theorem}

\pagebreak
\begin{proof}
Theorem~\ref{t-siren2sum} says that
$\alpha_\textbf{k}(a)=\prod_{i=1}^{n}J_{k_i}(a_i).$
To estimate an upper bound for this number, we use the following inequality~\cite{paris1984inequality}, which gives an upper bound for the Bessel functions $J_i(a)$.
\begin{align}\label{e-upper_bound_bessel_functions}
    \abs{J_i(a)}<\frac{\left(\frac{\abs{a}}{2}\right)^i}{i!}\,, \quad i>0,\quad a>0.
\end{align}
Observe that this inequality also holds for $a<0$ since $\abs{J_i(-a)}=\abs{J_i(a)}$.
Therefore, replacing Inequality~\ref{e-upper_bound_bessel_functions} in $\prod_{i=1}^{n}J_{k_i}(a_i)$ and using $\abs{J_{-i}(a)}=\abs{J_i(a)}$ results in the desired inequality~\ref{e-upper-bound-freq}.
\end{proof}

In particular, assuming $|a_i|\leq 1$ for $i=1,\ldots, n$, the upper bound provided by Theorem~\ref{t-upper_bound} implies that the amplitude $\alpha_\textbf{k}(a)$ associated with the harmonics $\sin\Big(\dot{\textbf{k}}{\omega x +\varphi}+ b\Big)$ is bounded as follows:
\begin{align}\label{e-upper_bound_bessel_functions_1}
\text{If} \quad \norm{a}_\infty\leq 1 \quad \text{then} \quad \abs{\alpha_\textbf{k}(a)}<\frac{1}{ \abs{k_1}!\cdots \abs{k_n}! \cdot 2^{\abs{k_1}+\cdots +\abs{k_n}}}.
\end{align}
This means that the frequencies $\dot{\textbf{k}}{\omega}$ produced by the sinusoidal neuron $h$ are bounded by a number that decreases rapidly as $\norm{\textbf{k}}_\infty$ increases.

In practice, the most successful initialization for the hidden weights $a$ considers values within the interval $\left(-\sqrt{\frac{6}{n}}, \sqrt{\frac{6}{n}}\right)$, as suggested in~\cite{sitzmann2020implicit}. Therefore, the upper bound provided in Inequality~\ref{e-upper_bound_bessel_functions_1} is applicable.

\vspace{0.2cm}

\subsection{Sorting the amplitudes}
\label{s-sorting}

In this section, we present a kind of sorting behavior of the amplitudes $\alpha_\textbf{k}(a)$ of the frequencies $\dot{\textbf{k}}{\omega}$ in terms of the integer vectors $\textbf{k}=(k_1,\ldots, k_n)$. 
However, before approaching the general case, let us consider $n=1$.
Specifically, we consider the case presented in
Equation~\ref{e-fourier-expansion-composition}:
$$\sin\big(a\sin(\omega x)\big)\!=\!\!\sum_{k\in\Z \text{ odd}} \!\!\!J_k(a) \sin(k\omega x)$$
Here, we are assuming $a\in\R$. We note that the absolute values of the Bessel functions $\norm{J_k(a)}$ exhibit a sorting pattern when $a$ is within the interval $\left(-\frac{\pi}{2}, \frac{\pi}{2}\right)$:
\begin{align}\label{e-sorting_of_bessel_functions}
    \text{If} \quad \abs{a}<\frac{\pi}{2} \text{ and } a\neq 0 \quad \text{then} \quad  \abs{J_1(a)}>\abs{J_2(a)}>\abs{J_3(a)}>\cdots
\end{align}
This behavior is shown in Figure~\ref{f-bessel_functions}. 
\begin{figure}[ht]
    \centering
    \includegraphics[width=0.9\columnwidth]{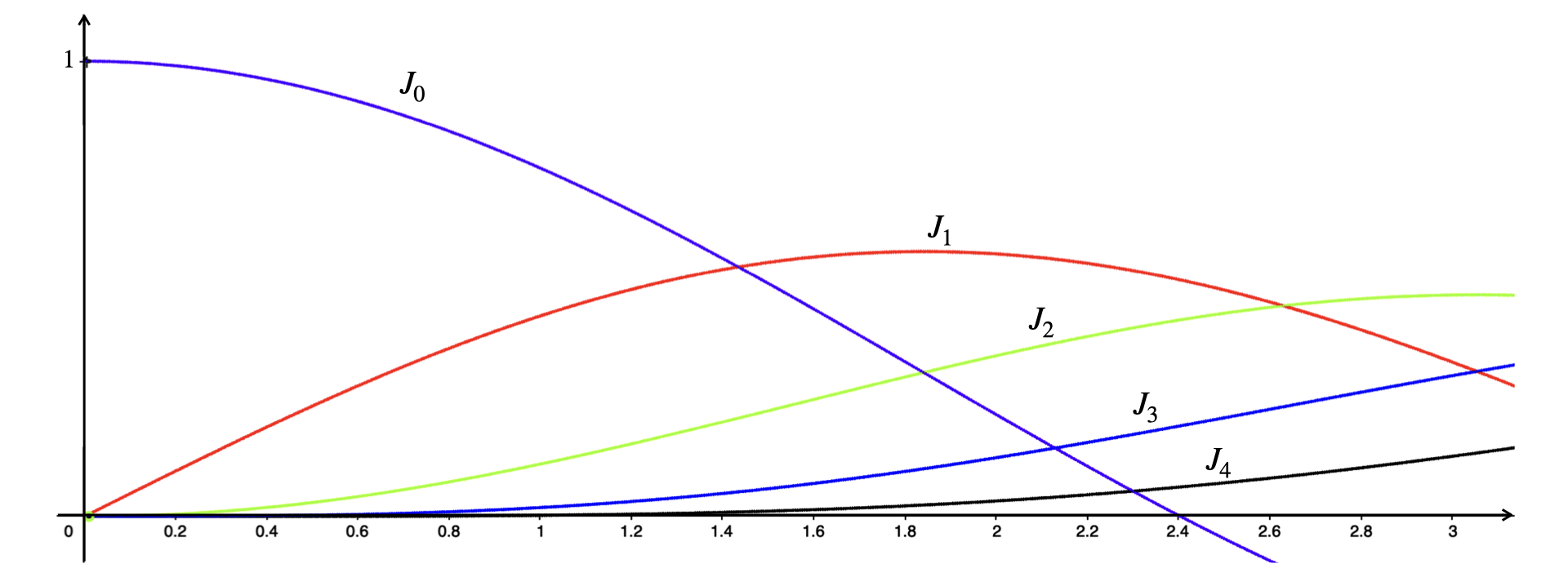}
    \vspace{-0.5cm}
    \caption{Bessel functions.}
    \label{f-bessel_functions}
\end{figure}

\noindent We also observe that $\abs{J_0(a)}>\abs{J_1(a)}$ for values of $a$ within an interval around $[0,1.4]$. However, we begin by proving the case outlined in Equation~\ref{e-sorting_of_bessel_functions}. We note that it is sufficient to show that $J_k$ is positive when $0< a\leq\frac{\pi}{2}$ and that $\frac{J_{k+1}(a)}{J_k(a)}<1$ for $k>0$.
This is because for $-\frac{\pi}{2}< a\leq 0$, we have $\abs{J_k(-a)}=\abs{J_k(a)}$.

The above facts are consequences of the following two inequalities presented~in~\cite{ifantis1990inequalities}:
\begin{align*}
    \frac{J_{k+1}(a)}{J_k(a)}&>\frac{a}{2k+2}>0,\quad  k\geq 0 \\
    \frac{J_{k+1}(a)}{J_k(a)}&<\frac{a}{2k+1}<1, \quad  k> 0.
\end{align*}

Equation~\ref{e-sorting_of_bessel_functions} implies that when $\abs{a}\leq\frac{\pi}{2}$ and $a\neq 0$, the amplitudes $J_k(a)$ of the expansion of $\sin\big(a\sin(x)\big)$ (Equation~\ref{e-fourier-expansion-composition}) are sorted in decreasing order by the index $k$. 
Additionally, since $\abs{J_{-k}(a)}=\abs{J_k(a)}$ we obtain a stronger result:
\begin{align}\label{e-sorting_of_bessel_functions_1}
    \text{Let } k,l\in \Z \text{ such that } \abs{k}<\abs{l} \text{ then }  \abs{J_k(a)}>\abs{J_l(a)}.
\end{align}

To include the case $k=0$ we must reduce the interval containing $a$. In Figure~\ref{f-bessel_functions} we can see that the intersection between the graphs of $J_0$ and $J_1$ happens before~$\frac{\pi}{2}$. 
We can extend the result of Equation~\ref{e-sorting_of_bessel_functions_1} to include the case $k=0$ if we consider $\abs{a}\leq 1$.
That is, $J_0(a)>J_1(a)$ for $\abs{a}<1$.
This is a consequence of the following recurrence relation given in \cite[Page 361]{abramowitz1964handbook}:
$$\frac{2k}{a}J_k(a)=J_{k-1}(a)+J_{k+1}(a).$$

As shown in Figure~\ref{f-bessel_functions}, the inequality can be extended to values near $\abs{a}\!<\!1.4$.
However, determining the exact location involves solving $J_0(x)-J_1(x)=0$, for which, to our knowledge, there is no closed-form solution.

\vspace{0.2cm}

We now return to the general case $n\ge 1$. For this, we present a generalization of the result in Equation~\ref{e-sorting_of_bessel_functions_1}.

\begin{theorem}\label{t-sorting}
Let $a\in\R^n$ such that $\abs{a_i}\leq 1$ and $a_i\neq 0$. Then, for every pair of distinct integer vectors $\textbf{k}, \textbf{l}\in \Z^n$ satisfying $\abs{k_i}\leq \abs{l_i}$, we have $\abs{\alpha_\textbf{k}(a)}>\abs{\alpha_\textbf{l}(a)}$.
\end{theorem}
\begin{proof}
By Theorem~\ref{t-siren2sum}, $\alpha_\textbf{k}(a)=\prod_{i=1}^{n}J_{k_i}(a_i)$. Since $\abs{a_i}\!\leq\! 1$, $a_i\!\neq\! 0$, and $\abs{k_i}\!\leq\!~\abs{l_i}$, Equation~\ref{e-sorting_of_bessel_functions_1} implies $\abs{J_{k_i}(a_i)}\geq\abs{J_{l_i}(a_i)}$.
Using this, we obtain the inequality:
\begin{align}
\abs{\alpha_\textbf{k}(a)}=\prod_{i=1}^{n}\abs{J_{k_i}(a_i)}>\prod_{i=1}^{n}\abs{J_{l_i}(a_i)}=\abs{\alpha_\textbf{l}(a)}.
\end{align}
The strict inequality is due to the fact that there is at least one $j\in \{1,\ldots, n\}$ such that $\abs{k_j}< \abs{l_j}$, which is a consequence of $\textbf{k}\neq \textbf{l}$ and $\abs{k_i}\leq \abs{l_i}$.
\end{proof}

In practice, we can apply Theorem~\ref{t-sorting} when $\abs{a_i}\leq 1.4$, which can be estimated numerically. Therefore, by keeping the hidden weights $a$ of the sinusoidal neuron $h$ bounded by $1.4$, we can effectively sort the amplitudes $\alpha_\textbf{k}(a)$ associated with the (new) frequencies $\dot{\textbf{k}}{\omega}$ based on the integer vectors $\textbf{k}$.

Moreover, Theorem~\ref{t-upper_bound} says that increasing $\abs{\textbf{k}}_\infty$ results in $\abs{J_\textbf{k}(a)}$ going to zero. 
Theorem~\ref{t-sorting} states that these amplitudes are also arranged in decreasing order. Thus, we can truncate the harmonic expansion of the neuron $h$ by summing up to index vectors $\textbf{k}$ satisfying $\abs{\textbf{k}}_\infty<B$, where $B$ serves as an upper bound. In this case, the least important frequencies $\dot{\textbf{k}}{\omega}$, for which $\norm{\textbf{k}}_{\infty}>B$, are discarded, and their corresponding amplitudes $\abs{J_\textbf{k}(a)}$ are small.

\section{Understanding sinusoidal networks with one hidden layer}
\label{s-sinusoidalMLP}
In this section, we extend Theorem~\ref{t-siren2sum} to a sinusoidal network with one hidden layer $f(x)=\textbf{l}\circ \textbf{h} \circ \textbf{s}(x)+d$; as defined in
Equation~\ref{e-siren}.

In Section~\ref{s-introduction}, we observed that the hidden sinusoidal layer $\textbf{h}$ receives a list of $n$ harmonics $\textbf{s}_j(x)=\sin(\omega_jx+\varphi_j)$ and outputs a list of $n$ sinusoidal neurons $\textbf{h}_i(x)=\sin\left(\sum_{j} a_{ij}\textbf{s}_j(x)+b_i\right)$, see Equation~\ref{e-second_layer}. Then, the linear layer $\textbf{l}$ combines these neurons resulting in the following expression:
\begin{align}
    f(x)=\sum_{i=1}^n c_i\sin\left(\sum_{j=1}^n a_{ij}\sin(\omega_ix+\varphi_i)+ b_j\right)+d
\end{align}
where $\textbf{a}=(a_1, \ldots, a_n)$, with $a_i,b\in\R^n$ and $d\in\R$.

Observe that applying Corollary~\ref{c-fourier_expansion} to each neuron $\textbf{h}_i$ we obtain the following sine-cosine expansion:
$$\textbf{h}_i(x)= \sum_{\textbf{k}\in\Z^n}A_\textbf{k}\cos\Big(\dot{\textbf{k}}{\omega} x\Big) + B_\textbf{k}\sin \Big(\dot{\textbf{k}}{\omega} x\Big).$$
We are going to use this formula to expand the whole network $f$ in a similar manner.  
We choose this expansion form because only the amplitudes $A_\textbf{k}$ and $B_\textbf{k}$ depend on the hidden weights $a_i$ and $b_i$. 
This allows us to rewrite (in the linear layer) the amplitudes associated with each frequency $\dot{\textbf{k}}{\omega}$ as a linear combination of the corresponding amplitudes of each neuron. 

\begin{theorem}\label{t-network2sum}
A sinusoidal network $f(x)=\textbf{l}\circ \textbf{h} \circ \textbf{s}(x)+d$ with one hidden layer and width $n$ expands as sum of sines and cosines:
\begin{align}\label{e-network_approx}
f(x)= \sum_{\textbf{k}\in\Z^n}\dot{c}{A_{\textbf{k}}}\cos\Big(\dot{\textbf{k}}{\omega} x\Big)+ \dot{c}{B_{\textbf{k}}}\sin\Big(\dot{\textbf{k}}{\omega} x\Big)+d,
\end{align}
with $A_{\textbf{k}}\!=\!\big(A_{\textbf{k}}(a_1, b_1),\ldots, A_{\textbf{k}}(a_n, b_n)\big)$ and $B_{\textbf{k}}\!=\!\big(B_{\textbf{k}}(a_1, b_1),\ldots, B_{\textbf{k}}(a_n, b_n)\big)$.
\end{theorem}
\begin{proof}
The proof consists of applying Corollary~\ref{c-fourier_expansion} to each sinusoidal neuron $\textbf{h}_i$ of $f$, which are the outputs of the hidden layer $\textbf{h}$.
\begin{align*}
    f(x)&=\sum_{i=1}^n c_i \underbrace{\sin\left(\sum_{j=1}^n a_{ij}\sin(\omega_jx+\varphi_j)+ b_i\right)}_{\textbf{h}_i}+d\\
        &=\sum_{i=1}^n c_i\left(\sum_{\textbf{k}\in\Z^n}A_\textbf{k}(a_i,b_i)\cos\Big(\dot{\textbf{k}}{\omega} x\Big) + B_\textbf{k}(a_i,b_i)\sin\Big(\dot{\textbf{k}}{\omega} x\Big)\right)+d\\
        &=\sum_{\textbf{k}\in\Z^n}\sum_{i=1}^n c_i A_{\textbf{k}}(a_i, b_i)\cos\Big(\dot{\textbf{k}}{\omega} x\Big)+ \sum_{i=1}^n c_i B_{\textbf{k}}(a_i, b_i)\sin\Big(\dot{\textbf{k}}{\omega} x\Big)+d\\
        &=\sum_{\textbf{k}\in\Z^n}\dot{c}{A_{\textbf{k}}(\textbf{a}, b)}\cos\Big(\dot{\textbf{k}}{\omega} x\Big)+ \dot{c}{B_{\textbf{k}}(a, b)}\sin\Big(\dot{\textbf{k}}{\omega} x\Big)+d.
\end{align*}
In the fourth equality we are defining $A_{\textbf{k}}(\textbf{a},{b})$ as vector $\big(A_{\textbf{k}}(a_1, b_1),\ldots, A_{\textbf{k}}(a_n, b_n)\big)$; the same for $B_{\textbf{k}}$.
\end{proof}

A key observation in the proof of Theorem~\ref{t-network2sum} is that the harmonics given by the expansion of each sinusoidal neuron $\textbf{h}_i(x)=\sin\left(\sum_{j} a_{ij}\textbf{s}_i(x)+b_i\right)$ are the same. 
In~other words, the expansion of each neuron $\textbf{h}_i$ is not producing any new frequency. Thus, if a given frequency is not represented in the expansion of $f$ (Equation~\ref{e-network_approx}), increasing the number of hidden neurons will not solve this problem.
This is a consequence of Theorem~\ref{t-siren2sum} which states that those harmonics depend only on the weights of the first layer.

\vspace{0.2cm}
Clearly, we can derive an exponential expansion form for the network~$f$: 
\begin{align*}
     f(x)=\sum_{\textbf{k}\in\Z^n}\dot{c}{c_\textbf{k}(\textbf{a},b)}\text{e}^{i\dot{\textbf{k}}{\omega x+\varphi}} &= \sum_{\textbf{k}\in\Z^n}\underbrace{\dot{c}{c_\textbf{k}(\textbf{a},b)}\text{e}^{i\dot{\textbf{k}}{\varphi}}}_{c_\textbf{k}(\textbf{a},b,c, \varphi)}\text{e}^{i\dot{\textbf{k}}{\omega} x}\\
     &= \sum_{\textbf{k}\in\Z^n}{c_\textbf{k}(\textbf{a},b,c, \varphi)}\text{e}^{i\dot{\textbf{k}}{\omega} x}
     .
\end{align*}
Observe that the weights $\textbf{a},b,c, \varphi$ determine the amplitudes of each frequency $\dot{\textbf{k}}{\omega}$ in the expansion.

\section{Applications}

\subsection{Training sinusoidal networks}
\label{s-training}

Let $f$ be a sinusoidal network $f$ with a single hidden layer, to train $f$ it is common to define a loss function and use the \textit{gradient descent algorithm} to approximate a minimum. \citet{parascandolo2016taming}
reported that we have to be careful in the parameter initialization of $f$ because the training can lead to an undesired local minimum.
\citet{sitzmann2020implicit} overcomes this by providing an initialization scheme that preserves the distributions of activations through the network layers.
It consists of choosing the weights of each hidden layer in $\left(-\sqrt{\frac{6}{n}}, \sqrt{\frac{6}{n}}\right)$, where $n$ is the width of $f$.
On the other hand, Theorem~\ref{t-network2sum} states that $f$ can be expressed as
$$f(x)= \sum_{\textbf{k}\in\Z^n}\dot{c}{A_{\textbf{k}}}\cos\Big(\dot{\textbf{k}}{\omega} x\Big)+ \dot{c}{B_{\textbf{k}}}\sin\Big(\dot{\textbf{k}}{\omega} x\Big)+d.$$
Then, using \textit{Cauchy–Schwarz} inequality, Theorem~\ref{t-upper_bound}, and  $\abs{a_{ij}}, \abs{c_j}\!\leq\! \sqrt{\frac{6}{n}}$, results~in:
\begin{align}\label{$e$-siren-inequality}
    \abs{\dot{c}{A_{\textbf{k}}}}, \abs{\dot{c}{B_{\textbf{k}}}}<
    \sqrt{6n}\frac{\left(\frac{3}{2n}\right)^{\frac{\abs{k_1}+\dots+\abs{k_n}}{2}}}{\abs{k_1}!\dots \abs{k_n}!}.
\end{align}

Equation~\ref{$e$-siren-inequality} gives an upper bound for the amplitude associated with the frequency $\dot{\textbf{k}}{\omega}$. Thus, if $n>1$ and $\norm{\textbf{k}}$ is large, the inequality implies that the amplitude is small.
For an example, consider the case $n=32$, that is, $f$ has $32$ neurons at each layer. We use Equation~\ref{$e$-siren-inequality} to compute the upper bounds for the amplitudes associated with the frequencies $\dot{\textbf{k}}{\omega}$ in the expansion of $f$.
\begin{figure}[ht]
    \begin{center}
    \includegraphics[width=0.7\textwidth]{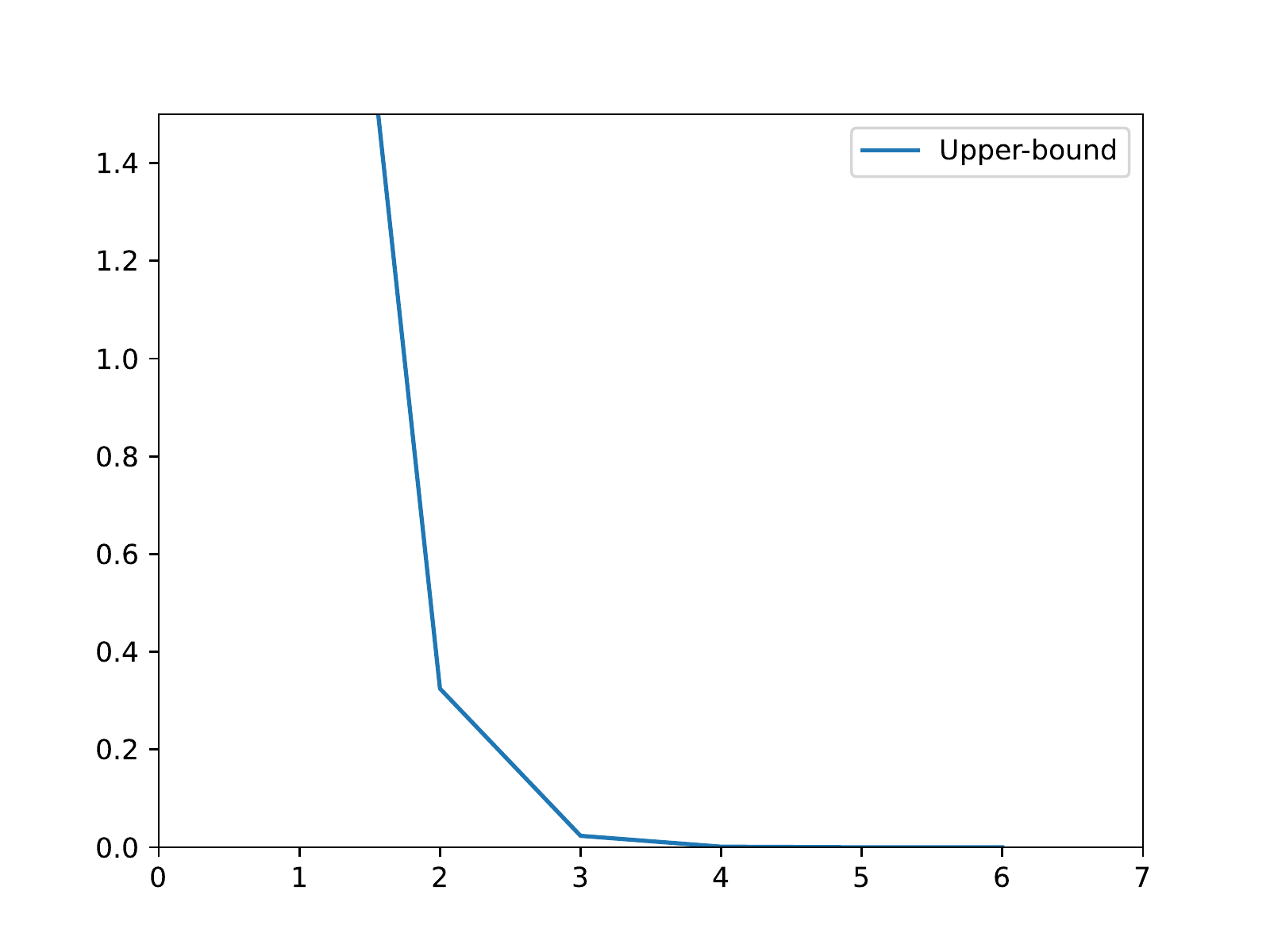}
    \caption{Upper bounds for the amplitudes associated to the frequencies $\dot{\textbf{k}}{\omega}$ in the case $n=32$. We consider $\textbf{k}=(i, 0,\ldots, 0)$ for $i=0,\ldots,7$ which coincides with the $x$-axis in the graph. In the case $i=7$ we obtained an upper-bound of $\approx 2e-06$}
    \label{f-upper-bound}
    \end{center}
\end{figure}

Figure~\ref{f-upper-bound} shows the upper bounds for the amplitudes associated with the frequencies $\dot{\textbf{k}}{\omega}$ with $\textbf{k}=(i, 0,\ldots, 0)$ for $i=0,\ldots,7$. 
Note that they decrease rapidly as we increase $i$.
Furthermore, it is evident that increasing the number of null entries in $\textbf{k}$ leads to even lower upper bounds.
Moreover, when we increase $n$, the amplitudes decrease even further due to the term $\frac{3}{2n}$ in Inequality~\ref{$e$-siren-inequality}.
This suggests that during training, the sinusoidal network will prioritize the lower frequencies.
Consequently, this provides a justification for the \textit{spectral bias phenomenon}~\cite{rahaman2019spectral}.

\subsection{Network initialization}
\label{s-initialization}
Let $g$ be a harmonic sum with frequencies $\tau_1,\ldots, \tau_K$. To~approximate $g$ by a sinusoidal network $f$ with one hidden layer we have to find its coefficients: 
$$f(x)=\sum_{i=1}^n c_i\sin\left(\sum_{j=1}^n a_{ij}\sin(\omega_ix+\varphi_i)+ b_j\right)+d.$$
First, we need to define its \textit{width} $n$, then, we have to initialize the coefficients to start the training using some variant of the \textit{gradient descent algorithm}. In the examples below we employed the ADAM \textit{algorithm}~\cite{kingma2014adam}.

\subsubsection{Network width}

In this section, we give a lower bound $M$ for the network width $n$. Observe that $K$ gives us an upper bound.  

Using Theorem~\ref{t-network2sum} we obtain a sinusoidal expansion of $f$ with frequencies $\dot{\textbf{k}}{\omega}$. On the other hand, Theorem~\ref{t-upper_bound} says that the amplitudes of such frequencies decrease rapidly as $\norm{\textbf{k}}_{\infty}$ grows. Therefore, we can truncate the expansion considering $\norm{\textbf{k}}_{\infty}\leq B$ for some small integer $B>0$.
As observed in Section~\ref{s-sinusoidal_neurons} the resulting truncated harmonic sum can represent up to $\frac{(2B+1)^n-1}{2}$.
Since we need to represent $K$ frequencies, we can assume that:
\begin{align}\label{e-width_formula}
    K=\frac{(2B+1)^n-1}{2}\quad \to \quad  n=\frac{\ln(2K+1)}{\ln(2B+1)}.
\end{align}
Equation~\ref{e-width_formula} gives a lower bound for the network width. For example, consider a signal $\gt{f}(x)=\sum_{k=1}^{12} c_k\sin\big(\phi_k x\big)$ with domain in $[-1,1]$ consisting of a sum of $12$ different sines with frequencies $\phi_k=2k+1$ and amplitudes $c_k$. 
Thus, assuming $B=2$, Equation~\ref{e-width_formula} says that we would need at least $n=2$ to define our network~$f$.
With such a width our network $f$ can be parameterized by only $13$ coefficients while $\gt{f}$ needs $24$ parameters.
Thus, $f$ compresses the expression of $\gt{f}$.

To evaluate the representation capacity of $f$, we examine three instances by manipulating the parameters of $\gt{f}$. Initially, we define $c_k$ as a decreasing sequence that ranges from $c_1=1$ to $c_{12}=0.005$.
Figure~\ref{f-testing_width_formula} (left) shows the result of training a network $f$ (shown in orange) with a width of $2$ to approximate the target signal $\gt{f}$ (shown in blue). As expected, by employing suitable initialization ($\omega_1$ and $\omega_2$ initialized with values in $[-5,5]$), we achieved an approximation with an error of approximately $\norm{f-\gt{f}}_2\approx 5e-6$.
\begin{figure}[ht]
    \begin{center}
    \vspace{-0.2cm}
    \includegraphics[width=0.328\textwidth]{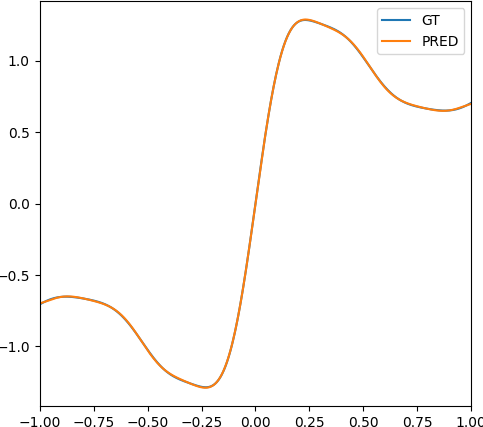}
    \includegraphics[width=0.328\textwidth]{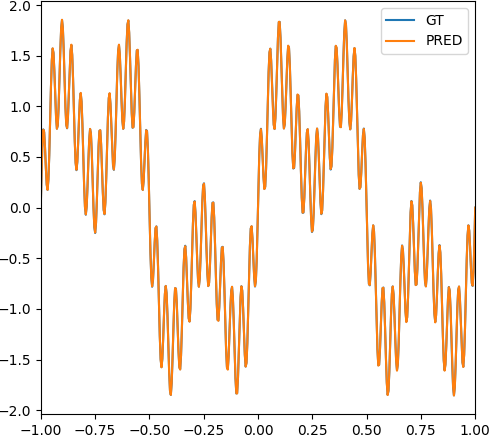}
    \includegraphics[width=0.325\textwidth]{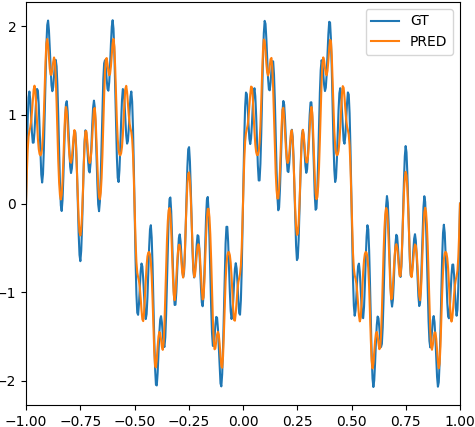}
    \vspace{-0.2cm}
    \caption{Graphs of the neural network $f$ with width $2$ trained to approximate signals consisting of the sum of $12$ sines with different frequencies and amplitudes.}
    \label{f-testing_width_formula}
    \end{center}
\end{figure}

In the second example, we employed frequencies $\phi_k=(2k+1)2\pi$, leading to the function $\gt{f}$ having two periods within the interval $[-1,1]$. Additionally, for the amplitudes, we selected higher amplitudes for the frequencies $\varphi_{11}$ and $\varphi_{12}$. To achieve this, we initialized $\omega_1=2\pi$ and $\omega_{2}=46\pi$, which remained \textit{frozen} during training. This yielded an approximation with an error of $\norm{f-\gt{f}}_2\approx 5e-5$ (Fig~\ref{f-testing_width_formula}~(middle)).
For the final example, we further increased the amplitudes of $\phi_5$ and $\phi_6$, while initializing $\omega_1=2\pi$ and $\omega_{2}=22\pi$. However, this resulted in an approximation with an error of approximately $\norm{f-\gt{f}}2\approx 0.07$ (Fig~\ref{f-testing_width_formula} (right)). We also attempted using $\omega{2}=46\pi$, but obtained a similar outcome. We overcome this issue by considering three neurons.



\subsubsection{Frequency initialization}
To find the input frequencies $\omega_i$ we use Equation~\ref{e-new_frequencies} which provides the frequencies $\beta_\textbf{k}(\omega)=\sum_{i=1}^{n} k_i\omega_{i}$ with $\norm{\textbf{k}}_{\infty}\leq B$, of the expansion of $f$ given by Theorem~\ref{e-network_approx}.
Let $\textbf{k}^1, \ldots, \textbf{k}^K$ be an enumeration of the integer vectors $\textbf{k}$ satisfying $\norm{\textbf{k}}_{\infty}\leq m$. 
We~use~Theorem~\ref{t-sorting} to sort these vectors with respect to the amplitudes of the frequencies $\tau_i$ in the ground truth signal $\gt{f}$. Therefore, a solution $\omega$ of the following problem is a candidate to be the set of input~frequencies.
$$\left\{
\begin{array}{cc}
     \tau_1&= k^1_1\omega_1+\dots+k^1_n\omega_n  \\
     \vdots&\\
     \tau_K&= k^K_1\omega_1+\dots+k^K_n\omega_n  
\end{array}\right.
\longleftrightarrow \left[
\begin{array}{c}
     \tau_1\\
     \vdots\\
     \tau_K  
\end{array}\right]
=
\left[k^i_j\right]
\left[
\begin{array}{c}
     \omega_1\\
     \vdots\\
     \omega_n  
\end{array}\right]
$$
As $K>n$, we can use the \textit{least-squares} method to find a solution approximation.
The biases (phase shifts) $\varphi_i$ can be calculated in a similar way using Equation~\ref{e-new_frequencies}.

\subsection{On the periodicity of sinusoidal networks}
\label{s-periodicity}

A \textit{periodic function} $\gt{f}:\R\to \R$ is a function satisfying $\gt{f}(x) \!=\! \gt{f}(x + P)$ where $P$ is the \textit{period}.
For an example, consider $h(x)\!=\!a\sin(\omega x+  \varphi)$ where $\omega\!=\!l\frac{2\pi}{P}$, with $l\in\Z$, is the \textit{frequency}, $a$ is the \textit{amplitude}, and $\varphi$ is the \textit{phase shift}.
Thus, we can force the first layer of a sinusoidal network to be periodic.

Specifically, let $f:\R\to \R$ be a sinusoidal network, defining $\omega_{i}=l_{i}\frac{2\pi}{P}$, with $l_{i}\in\Z$, implies that input neurons $\textbf{s}_i\!:\!\R\!\to\! \R^{n}$ are periodic functions with period $P$:
\begin{align}
\displaystyle
    \textbf{s}(x)=
    \left(
    \begin{array}{c}
        \sin\big(l_{1}\frac{2\pi}{P}{x}+\varphi_1\big)\\
         {\footnotesize\vdots}\\
         \sin\left(l_{n}\frac{2\pi}{P}{x}+\varphi_{n}\right)
    \end{array}
    \right).
\end{align}

We now prove that the composition $\textbf{h}\big(\textbf{s}(x)\big)$ outputs a list of neurons $\textbf{h}_i$ that are periodic with period $P$. For this, we use Theorem~\ref{t-siren2sum} to rewrite $\textbf{h}_i$ as follows:
\begin{align*}
\textbf{h}_i(x)=\sum_{\textbf{k}\in\Z^n}\alpha_\textbf{k}(a)\sin\left(\frac{2\pi}{P}\dot{\textbf{k}}{\textbf{l}} x +\dot{\textbf{k}}{\varphi}+ b\right).
\end{align*}
Since $f(x)=\sum_i c_i \textbf{h}_i(x)+d$, the network $f$ is also periodic with period $P$, and we have proved the following result.
\begin{theorem}
\label{t-periodic}
    If the input neurons of a sinusoidal network $f$ has its weights expressed by $\omega_{i}=l_{i}\frac{2\pi}{P}$, with $l_i\in\Z$, then $f$ is periodic with period $P$.
\end{theorem}

We define a \textit{periodic sinusoidal network} as a sinusoidal network such that its input neurons are periodic. 
To provide an example, we train a periodic network with period $2$ to learn a \textit{square wave} with period $2$ and amplitude $0.5$ defined in the interval $[-1,1]$ as follows.
\begin{align*}
\gt{f}(x)=
\begin{cases}
    0.5 \quad &\text{if} \quad -1\leq x\leq 0\\
    -0.5 \quad &\text{if} \quad\quad 0\leq x\leq 1
\end{cases}
\end{align*}
We recall that its Fourier series is given by $\gt{f}(x)\!=\!-\!\sum_{i=1}^\infty \frac{2}{i\pi}\sin\big(i\pi x\big)$, with $i$ odd. 
Figure~\ref{f-pulse_train} (left) illustrates the sum of the first $92$ harmonics of this series. 

We assume the periodic sinusoidal network $f$ to have width $5$. To fit $f$ to $\gt{f}$, we initialize the weights of the first layer of the network with the five most important frequencies $\pi[1,3,5,7,9]$ of the Fourier series of $\gt{f}$.
The weights $\omega$ are kept fixed during training (not optimized).

Figure~\ref{f-pulse_train} (right) shows the graph of the trained network. Observe that the \textit{Gibbs phenomenon}
is minimized in the network approximation, which is not the case in the truncated Fourier series (left).
\begin{figure}[ht]
    \begin{center}

    \begin{tabular}{cc}
        {\small Truncated Fourier series} & {\small Periodic sinusoidal network} \\
        {\small with $92$ coefficients} & {\small with $46$ coefficients }\\
         \includegraphics[width=0.47\textwidth]{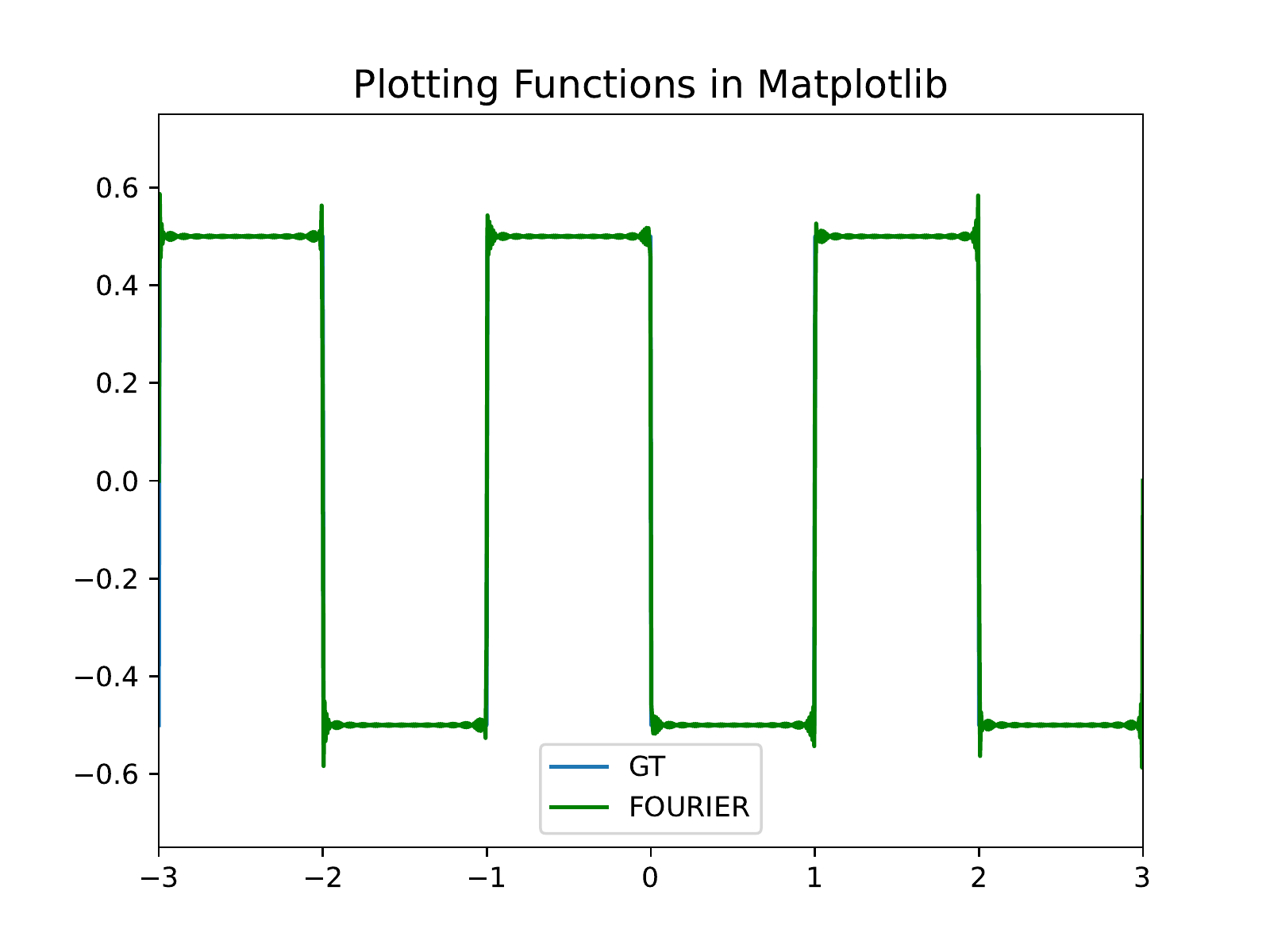}& \includegraphics[width=0.47\textwidth]{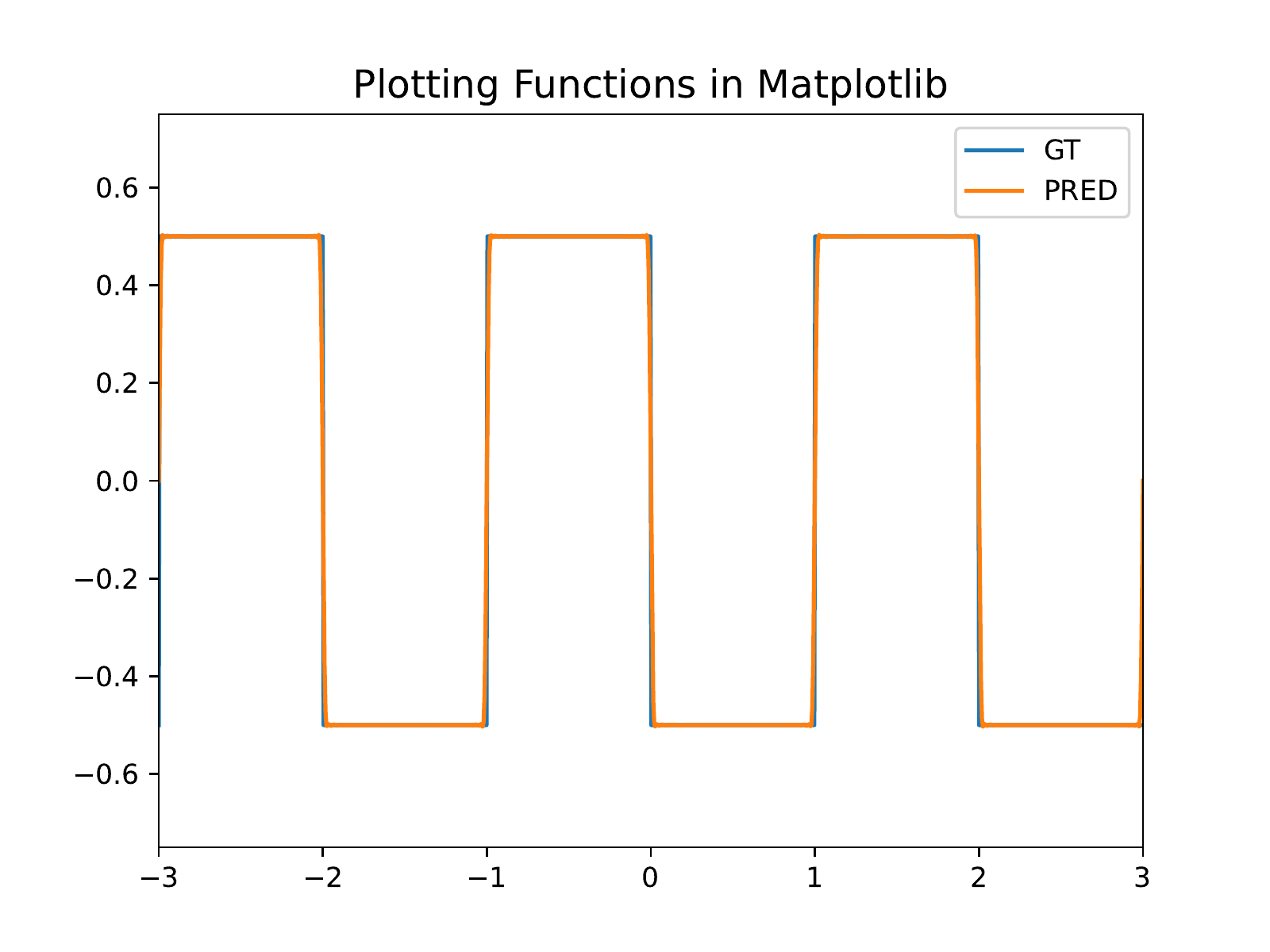}
    \end{tabular}

    \end{center}
    \caption{Approximations of the square wave. On its left, we have its truncated Fourier series consisting of the sum of its first $92$ terms. On the left, we have the graph of a periodic neural network $f$ with a width $5$ trained only in the period $[-1,1]$ of $\gt{f}$.}
    \label{f-pulse_train}
\end{figure}

There are two observations regarding the training of the periodic network~$f$:
\begin{itemize}
    \item First, we only have to consider the data inside a given period of the ground-truth function $\gt{f}$. In the example, we used only the values of $\gt{f}$ in the interval $[-1,1]$. Since $f$ has period $2$, the learned signal in $[-1,1]$ is translated to the real line, see Figure~\ref{f-pulse_train} (right);
    \item Second, the ground-truth $\gt{f}$ contains discontinuities at the integers. On the other hand, our sinusoidal network $f$ is a smooth function. We avoid possible representation inconsistencies by removing a small neighborhood of the discontinuities of $\gt{f}$. This left a space for $f$ to approximate $\gt{f}$ at such points smoothly.
\end{itemize}









\bibliographystyle{plainnat}
\bibliography{deep_implicits}

\end{document}